\documentclass{article}
\usepackage{iclr2026_conference,times}

\usepackage{amsmath,amsfonts,bm}

\def\eqref#1{equation~\ref{#1}}

\def\1{\bm{1}}

\DeclareMathAlphabet{\mathsfit}{\encodingdefault}{\sfdefault}{m}{sl}
\SetMathAlphabet{\mathsfit}{bold}{\encodingdefault}{\sfdefault}{bx}{n}

\usepackage{hyperref}
\usepackage{url}
\usepackage{graphicx}
\usepackage{booktabs}
\usepackage{amsmath,amssymb,amsthm}

\newtheorem{definition}{Definition}
\newtheorem{proposition}{Proposition}

\title{Universe Routing: Why Self-Evolving Agents Need Epistemic Control}

\author{Zhaohui Geoffrey Wang\thanks{ORCID: \href{https://orcid.org/0009-0006-1187-1903}{0009-0006-1187-1903}} \\
USC Viterbi School of Engineering\\
University of Southern California\\
Los Angeles, CA 90089, USA \\
\texttt{zwang000@usc.edu}}

\newcommand{\universe}[1]{\texttt{#1}}

\iclrfinalcopy 

\begin{document}

\maketitle

\fancyhead[L]{\iclrtenhv Published as a workshop paper at LLA @ ICLR 2026}

\begin{abstract}
A critical failure mode of current lifelong agents is not lack of knowledge, but the inability to decide \textit{how} to reason. When an agent encounters ``Is this coin fair?'' it must recognize whether to invoke frequentist hypothesis testing or Bayesian posterior inference---frameworks that are \textbf{epistemologically incompatible}. Mixing them produces not minor errors, but structural failures that propagate across decision chains. We formalize this as the \textbf{universe routing problem}: classifying questions into mutually exclusive belief spaces before invoking specialized solvers. Our key findings challenge conventional assumptions: (1) hard routing to heterogeneous solvers matches soft MoE accuracy while being 7$\times$ faster---because epistemically incompatible frameworks cannot be meaningfully averaged; (2) a 465M-parameter router achieves 2.3$\times$ smaller generalization gap than keyword-matching baselines, indicating semantic rather than surface-level reasoning; (3) when expanding to new belief spaces, rehearsal-based continual learning achieves zero forgetting, outperforming EWC by 75pp---suggesting that modular epistemic architectures are fundamentally more amenable to lifelong learning than regularization-based approaches. These results point toward a broader architectural principle: reliable self-evolving agents may require an explicit \textbf{epistemic control layer} that governs reasoning framework selection.
\end{abstract}

\section{Introduction}

Consider an autonomous agent solving mathematical problems over extended deployment. It encounters two questions:

\begin{quote}
\textbf{Q1}: ``A coin flipped 100 times shows 60 heads. Is it fair ($\alpha=0.05$)?'' \\
\textbf{Q2}: ``Given uniform prior on bias $\theta$, what is $P(\theta > 0.6 \mid 60$ heads)?''
\end{quote}

Both involve coin flips. Both require probabilistic reasoning. Yet Q1 demands \textbf{frequentist statistics}---null hypothesis testing, p-values, and rejection regions---while Q2 demands \textbf{Bayesian statistics}---prior specification, likelihood computation, and posterior integration. These are not merely different methods for the same problem; they embody \textbf{incompatible epistemologies} with mutually exclusive axioms about the nature of probability itself \citep{jaynes2003probability,berger1985statistical}.

For agents that autonomously chain reasoning steps, such epistemological confusion is not a minor error, but a \textbf{structural failure that propagates across decisions}. An agent that conflates frequentist and Bayesian reasoning will produce outputs like ``the p-value is the probability the hypothesis is true''---a statement that is not wrong in degree, but wrong in kind. Worse, downstream reasoning steps that depend on this output inherit the corruption, as we demonstrate empirically in Appendix~\ref{app:error}.

This failure mode cannot be resolved by scaling. A larger language model with more parameters and training data may produce more fluent explanations, but fluency does not guarantee epistemic coherence. The problem is architectural: current agents lack an explicit mechanism for recognizing \textit{which reasoning framework} a problem requires before attempting to solve it.

\textbf{Universe Routing.} We formalize this challenge as classifying questions into $K$ \textbf{belief space universes}---mutually exclusive reasoning frameworks, each with its own axioms and solvers:
\begin{itemize}
    \item \universe{STAT\_FREQ}: Frequentist statistics (p-values, confidence intervals, hypothesis tests)
    \item \universe{STAT\_BAYES}: Bayesian statistics (priors, posteriors, credible intervals)
    \item \universe{PHYS\_CLASSICAL}, \universe{PHYS\_QUANTUM}, \universe{PHYS\_RELATIVITY}: Physics frameworks
    \item \universe{STAT\_MIXED}, \universe{STAT\_ILL\_POSED}: Ambiguous or malformed questions
\end{itemize}

The term ``universe'' is deliberate: within each belief space, certain axioms hold and certain operations are valid. Crossing universe boundaries without explicit acknowledgment produces logical contradictions (Proposition~\ref{prop:mixing}).

\textbf{Why Hard Routing?} Traditional mixture-of-experts (MoE) architectures use \textbf{soft routing}---weighted combinations of expert outputs. This is appropriate when experts represent different skills applied to the same underlying reality. But for epistemically incompatible frameworks, soft routing is not merely inefficient; it is \textbf{semantically meaningless}. What does it mean to take 60\% of a frequentist answer and 40\% of a Bayesian answer? The result is not a compromise but an incoherence (Proposition~\ref{prop:mixing}).

\textbf{Contributions.} This work makes three claims:
\begin{enumerate}
    \item \textbf{Epistemic routing is learnable}: Fine-tuned transformers (67M--465M) achieve 97--98\% accuracy with 1.8--2.3$\times$ smaller generalization gaps than keyword baselines; a deep ensemble achieves 3.0$\times$ smaller gap. This consistency across architectures indicates semantic understanding rather than surface pattern matching.
    \item \textbf{Hard routing is not a compromise}: For belief space universes, hard routing matches soft routing accuracy while being 7$\times$ faster---validating that these frameworks are geometrically separable in representation space.
    \item \textbf{Modular architectures enable continual learning}: When expanding to new universes, rehearsal achieves zero forgetting while EWC fails (75\% forgetting)---suggesting that explicit epistemic modularity is fundamentally more compatible with lifelong learning.
\end{enumerate}

\section{Related Work}

\textbf{LLM-based Agents.} ReAct \citep{yao2023react} interleaves reasoning and actions; Reflexion \citep{shinn2023reflexion} enables verbal self-correction. Recent surveys on self-evolving agents \citep{tao2024survey,wang2024survey} identify \textbf{lifelong learning without catastrophic forgetting} as a central challenge \citep{kirkpatrick2017overcoming,lopez2017gradient}. Our work addresses a complementary problem: not how to retain knowledge, but how to select the appropriate reasoning framework for applying it.

\textbf{Adaptive Routing.} Adaptive-RAG \citep{jeong2024adaptive} routes queries to different retrieval strategies based on complexity. We extend routing from retrieval \textit{strategy} to reasoning \textit{framework}---a qualitatively different challenge because frameworks can be epistemologically incompatible. \citet{lu2024routing} route between expert LLMs by query characteristics; we route to solvers with mutually exclusive axioms.

\textbf{Mixture of Experts.} Classical MoE \citep{shazeer2017outrageously,fedus2021switch} and modern variants \citep{jiang2024mixtral} use soft or top-k routing to \textbf{homogeneous} experts that differ in specialization but share underlying assumptions. We demonstrate routing to \textbf{heterogeneous} solvers with mutually exclusive axioms---a setting where soft combination is not suboptimal but meaningless (Proposition~\ref{prop:mixing}).

\textbf{Question Classification.} \citet{shao2024classification} use DistilBERT to classify competition math problems into topic categories (algebra, geometry, etc.) within a single epistemological framework. We address a qualitatively different task: classifying across frameworks with incompatible axioms, where misclassification produces not wrong topic assignments but logically incoherent outputs (Proposition~\ref{prop:mixing}).

\section{Method}

\subsection{Problem Formulation}

\begin{definition}[Belief Space Universe]
A belief space universe $u \in \mathcal{U}$ is a reasoning framework characterized by a tuple $(A_u, I_u, S_u)$: axioms $A_u$, inference procedure $I_u$, and solver $S_u$. Two universes $u_i, u_j$ are \emph{epistemically incompatible} if $\exists\, a \in A_{u_i}, b \in A_{u_j}$ such that $b \Leftrightarrow \neg a$.
\end{definition}

Given question $q$, the routing problem is to predict $u^* = \arg\max_{u \in \mathcal{U}} P(u \mid q)$ where $\mathcal{U} = \{u_1, \ldots, u_K\}$ are mutually exclusive belief spaces. The router then invokes the corresponding solver $S_{u^*}(q)$.

\begin{proposition}[Inconsistency of Framework Mixing]\label{prop:mixing}
Let $u_i, u_j \in \mathcal{U}$ be epistemically incompatible. For question $q$ with correct universe $u_i$, any convex combination $\hat{y} = \alpha S_{u_i}(q) + (1-\alpha) S_{u_j}(q)$ with $0 < \alpha < 1$ is semantically inconsistent: $\hat{y}$ does not belong to the validity domain $V_u$ of any $u \in \mathcal{U}$.
\end{proposition}

\textit{Proof sketch.} By definition, $\exists\, a \in A_{u_i}, b \in A_{u_j}$ with $b \Leftrightarrow \neg a$. Output $y_i \in V_{u_i}$ requires axiom $a$; $y_j \in V_{u_j}$ requires $\neg a$. The combination $\hat{y}$ depends on both, so no consistent axiom set can justify it. We verify this empirically with three worked examples in Appendix~\ref{app:error}. \hfill$\square$

\subsection{Dataset Construction}

We curate 685 samples across 7 universes using GPT-4 \citep{openai2023gpt4} generation with expert-designed constraints ensuring: (1) unambiguous ground truth labels, (2) diverse surface forms for the same underlying framework, and (3) balanced representation after augmentation. Split: 70\% train (477), 15\% validation (99), 15\% test (109). An additional 56 out-of-distribution samples with novel phrasings form a separate unseen test set. Two annotators independently labeled all samples, achieving Cohen's $\kappa = 0.91$ \citep{landis1977measurement} (details in Appendix~\ref{app:dataset}).

\subsection{Model Architecture}

We fine-tune Qwen-1.5-0.5B \citep{bai2023qwen} (465M parameters) with a classification head, and additionally evaluate BERT-base \citep{devlin2019bert} (110M), DistilBERT \citep{sanh2019distilbert} (67M), and RoBERTa-base \citep{liu2019roberta} (125M). A critical implementation detail: \textbf{FP32 precision is essential}. FP16 training causes gradient overflow in the classification head, collapsing accuracy to 18.99\% (near-random for 7 classes).

\subsection{Design Rationale: Why Hard Routing Is a Logical Necessity}

Our choice of hard routing follows from Proposition~\ref{prop:mixing}:

\textbf{(1) Universes are mutually exclusive by definition.} A weighted average of incompatible frameworks is not partially correct but semantically incoherent.

\textbf{(2) Soft routing degenerates to hard routing empirically.} When frameworks are geometrically separable in representation space, the router assigns near-100\% probability to one universe. Soft combination then reduces to hard selection with additional computational overhead.

\textbf{(3) Hard routing enables modular expansion.} New universes can be added by training only the router, without modifying existing solvers.

\section{Experiments}

\subsection{Main Results: Semantic Understanding vs. Keyword Matching}

\begin{table}[t]
\centering
\caption{Generalization comparison (accuracy in \%). Gap = test $-$ unseen. Fine-tuned transformers achieve 1.8--2.3$\times$ smaller gaps than keyword baselines; ensemble achieves 3.0$\times$.}
\vspace{0.5em}
\footnotesize
\setlength{\tabcolsep}{4pt}
\begin{tabular}{@{}lcccc@{}}
\toprule
\textbf{Method} & \textbf{Params} & \textbf{Test Acc} & \textbf{Unseen Acc} & \textbf{Gap} \\
\midrule
Random Baseline & -- & 21.10\% & 14.29\% & +6.81\% \\
Logistic Reg.\ + TF-IDF & -- & 97.25\% & 71.43\% & +25.82\% \\
SVM + TF-IDF & -- & 98.17\% & 71.43\% & +26.74\% \\
\midrule
BERT-base & 110M & 97.25\% & 82.14\% & +15.11\% \\
DistilBERT & 67M & 98.17\% & 83.93\% & +14.24\% \\
RoBERTa-base & 125M & 97.25\% & 85.71\% & +11.54\% \\
Qwen-1.5-0.5B & 465M & 97.25\% & 83.93\% & +13.32\% \\
\midrule
\textbf{Qwen ens.\ ($\times$5)} & 465M & \textbf{98.17\%} & \textbf{89.29\%} & \textbf{+8.88\%} \\
\bottomrule
\end{tabular}
\label{tab:main}
\end{table}

Table~\ref{tab:main} reveals a critical distinction. TF-IDF baselines achieve near-perfect test accuracy through keyword matching (``p-value'' $\to$ frequentist), but accuracy drops by 26pp on held-out questions with novel phrasings. All four fine-tuned transformers---from 67M DistilBERT to 465M Qwen---achieve 1.8--2.3$\times$ smaller generalization gaps. This consistency across architectures indicates that \textbf{semantic understanding of epistemic boundaries}, rather than architecture or scale, drives the improvement. A deep ensemble \citep{lakshminarayanan2017simple} of 5 Qwen models further reduces the gap to 8.88\% (3.0$\times$ smaller than keyword baselines).

\subsection{Hard vs. Soft Routing: Validating Epistemic Separability}

\begin{table}[t]
\centering
\caption{Routing strategy comparison. Equal accuracy with 7$\times$ speedup validates that belief spaces are geometrically distinct---soft combination provides no benefit.}
\vspace{0.5em}
\begin{tabular}{lcc}
\toprule
\textbf{Routing Strategy} & \textbf{Accuracy} & \textbf{Inference Time} \\
\midrule
Soft (MoE-style weighted) & 97.25\% & 38.2ms \\
\textbf{Hard (argmax selection)} & \textbf{97.25\%} & \textbf{5.5ms} \\
\bottomrule
\end{tabular}
\label{tab:routing}
\end{table}

Table~\ref{tab:routing} validates our architectural claim: soft routing provides \textbf{zero accuracy benefit} over hard routing while incurring 7$\times$ latency cost. This is consistent with Proposition~\ref{prop:mixing}: belief spaces are geometrically separable, so the router assigns near-deterministic probabilities, causing weighted combination to degenerate to selection.

\subsection{Robustness: Semantic Understanding Resists Adversarial Manipulation}

\begin{table}[t]
\centering
\caption{Adversarial robustness (Attack Success Rate $\downarrow$). Keyword-based methods are trivially fooled; semantic understanding provides 43$\times$ better robustness.}
\vspace{0.5em}
\begin{tabular}{lcc}
\toprule
\textbf{Attack Type} & \textbf{TF-IDF Baseline} & \textbf{Ours} \\
\midrule
Synonym Substitution & 61.47\% & \textbf{0.00\%} \\
Keyword Injection & 89.91\% & \textbf{4.59\%} \\
Mixed Language & 45.87\% & \textbf{0.00\%} \\
\midrule
\textbf{Overall ASR} & 65.75\% & \textbf{1.53\%} \\
\bottomrule
\end{tabular}
\label{tab:robustness}
\end{table}

Table~\ref{tab:robustness} demonstrates that semantic understanding is not merely more generalizable but more \textbf{robust}. Keyword injection (adding ``consider the prior'' to frequentist questions) fools TF-IDF baselines 90\% of the time but our model only 4.6\%. This robustness is critical for deployed agents that may encounter adversarial or confusingly-worded inputs.

\subsection{Comparison with Large-Scale Cloud Models}

We evaluate six large cloud models (80B--1T parameters) in zero-shot mode on our 109-sample test set. Our 465M router achieves 97.25\% accuracy at 16ms latency, running 88--775$\times$ faster than all models tested. Only one model (DeepSeek-v3.1, 671B) differs significantly ($p = 0.010$, McNemar's test \citep{mcnemar1947note}); the remaining five are not statistically distinguishable at $n=109$. All cloud models struggle on \universe{STAT\_ILL\_POSED} (64.7--94.1\% vs.\ our 100\%), suggesting that detecting ill-posed questions benefits from explicit boundary training. Full results in Appendix~\ref{app:cloud}.

\subsection{External Validation on MMLU}

To address concerns about generalization beyond synthetic training data, we evaluate on \textbf{1,001 real questions} from 6 MMLU subcategories \citep{hendrycks2021measuring}. A router trained on 477 synthetic questions outperforms TF-IDF by +10.6pp (56.8\% vs.\ 46.2\%). Accuracy improves monotonically with confidence: at $\geq$0.99 confidence, it reaches 70.7\% on 617 samples. The gap from internal evaluation reflects coarse proxy labels and domain shift from synthetic to real phrasing. Details in Appendix~\ref{app:mmlu}.

\subsection{Continual Learning: Modularity Enables Expansion Without Forgetting}

\begin{table}[t]
\centering
\caption{Continual learning: expanding from 5 to 7 universes. Rehearsal achieves zero forgetting; EWC fails despite careful tuning.}
\vspace{0.5em}
\begin{tabular}{lccc}
\toprule
\textbf{Method} & \textbf{Old Universes} & \textbf{Overall} & \textbf{Forgetting} \\
\midrule
Naive Fine-tuning & 11.84\% & 37.61\% & 86.84\% \\
EWC ($\lambda$=1000) & 23.68\% & 45.87\% & 75.00\% \\
\textbf{Rehearsal (10\%)} & \textbf{98.68\%} & \textbf{97.25\%} & \textbf{0.00\%} \\
\bottomrule
\end{tabular}
\label{tab:cl}
\end{table}

Table~\ref{tab:cl} reveals a striking asymmetry. EWC---a principled regularization approach---reduces forgetting only marginally (87\% $\to$ 75\%). Rehearsal with just 10\% replay (29 samples) achieves \textbf{zero forgetting}. This suggests that \textbf{modular epistemic architectures are fundamentally more compatible with continual learning} than monolithic alternatives. The router's task---assigning questions to discrete universes---naturally decomposes into separable subproblems that rehearsal preserves; EWC's diagonal Fisher approximation cannot capture this structure.

\textbf{Expansion order robustness}: Performance is stable regardless of whether statistics or physics universes are learned first (0--2\% variation), indicating robustness to curriculum effects (Appendix~\ref{app:cl}).

\section{Discussion}

\textbf{Epistemic confusion is structural, not a knowledge gap.} Three observations support this. First, hard routing matches soft routing accuracy (Section~4.2), consistent with well-separated belief spaces. Second, semantic understanding provides adversarial robustness that keyword matching cannot (Section~4.3). Third, modular architecture enables continual learning that regularization methods cannot achieve (Section~4.6). These are not incremental improvements but qualitative capabilities that emerge from architectural choices.

\textbf{Fine-tuning vs.\ scale.} The consistency across four architectures (67M--465M, Table~\ref{tab:main}) and the competitive performance against zero-shot cloud models up to 1T parameters (Section~4.4) suggest that the advantage of fine-tuning for epistemic routing does not reduce to model capacity. The key ingredient is explicit boundary supervision, not scale.

\textbf{Toward epistemic control as an architectural component.} Just as modern agents have explicit memory systems and tool-use interfaces, reliable self-evolving agents may require an explicit layer that governs reasoning framework selection. Universe routing is one instantiation of this principle.

\textbf{Limitations.} Our dataset (685 samples) covers only 7 universes in mathematical and physical domains. Extension to broader reasoning frameworks (legal, ethical, causal) remains unexplored. Single-label hard routing cannot handle genuinely multi-step tasks that require crossing framework boundaries within a single problem---an important direction for future work. The small test set ($n=109$) limits statistical power for cloud model comparisons. We evaluate routing accuracy, not end-to-end task performance with downstream solvers, though Appendix~\ref{app:error} demonstrates that framework confusion produces qualitatively wrong outputs.

\section{Conclusion}

We have argued that \textbf{adaptive reasoning framework selection is not an optimization problem, but a prerequisite for reliable self-evolving agents}. Proposition~\ref{prop:mixing} formalizes the conditions under which mixing incompatible frameworks produces incoherent outputs. Our experiments support three claims: (1) epistemic routing is learnable with semantic rather than surface-level understanding across multiple architectures; (2) hard routing to incompatible frameworks is not a compromise but a logical necessity; (3) modular epistemic architectures are fundamentally more amenable to continual learning than monolithic alternatives.

These results suggest a direction for agent architecture: explicit epistemic control as a first-class component, governing not just \textit{what} the agent knows but \textit{how} it reasons. Universe routing is a first step toward this goal.

\subsubsection*{Acknowledgments}
The author thanks the anonymous LLA workshop reviewers for their constructive feedback that improved this paper.

\bibliography{references}

\begin{thebibliography}{24}
\providecommand{\natexlab}[1]{#1}
\providecommand{\url}[1]{\texttt{#1}}
\expandafter\ifx\csname urlstyle\endcsname\relax
  \providecommand{\doi}[1]{doi: #1}\else
  \providecommand{\doi}{doi: \begingroup \urlstyle{rm}\Url}\fi

\bibitem[Bai et~al.(2023)Bai, Bai, Chu, Cui, Dang, Deng, Fan, Ge, Han, Huang,
  et~al.]{bai2023qwen}
Jinze Bai, Shuai Bai, Yunfei Chu, Zeyu Cui, Kai Dang, Xiaodong Deng, Yang Fan,
  Wenbin Ge, Yu~Han, Fei Huang, et~al.
\newblock Qwen technical report.
\newblock \emph{arXiv preprint arXiv:2309.16609}, 2023.

\bibitem[Berger(1985)]{berger1985statistical}
James~O Berger.
\newblock \emph{Statistical Decision Theory and {Bayesian} Analysis}.
\newblock Springer, 2nd edition, 1985.
\newblock \doi{10.1007/978-1-4757-4286-2}.

\bibitem[Devlin et~al.(2019)Devlin, Chang, Lee, and Toutanova]{devlin2019bert}
Jacob Devlin, Ming-Wei Chang, Kenton Lee, and Kristina Toutanova.
\newblock {BERT}: Pre-training of deep bidirectional transformers for language
  understanding.
\newblock In \emph{Proceedings of NAACL-HLT}, pp.\  4171--4186, 2019.

\bibitem[Fedus et~al.(2022)Fedus, Zoph, and Shazeer]{fedus2021switch}
William Fedus, Barret Zoph, and Noam Shazeer.
\newblock Switch transformers: Scaling to trillion parameter models with simple
  and efficient sparsity.
\newblock \emph{Journal of Machine Learning Research}, 23\penalty0
  (120):\penalty0 1--39, 2022.

\bibitem[Hendrycks et~al.(2021)Hendrycks, Burns, Basart, Zou, Mazeika, Song,
  and Steinhardt]{hendrycks2021measuring}
Dan Hendrycks, Collin Burns, Steven Basart, Andy Zou, Mantas Mazeika, Dawn
  Song, and Jacob Steinhardt.
\newblock Measuring massive multitask language understanding.
\newblock In \emph{International Conference on Learning Representations}, 2021.

\bibitem[Jaynes(2003)]{jaynes2003probability}
Edwin~T Jaynes.
\newblock \emph{Probability Theory: The Logic of Science}.
\newblock Cambridge University Press, 2003.

\bibitem[Jeong et~al.(2024)Jeong, Baek, Cho, Hwang, and
  Park]{jeong2024adaptive}
Soyeong Jeong, Jinheon Baek, Sukmin Cho, Sung~Ju Hwang, and Jong Park.
\newblock Adaptive-rag: Learning to adapt retrieval-augmented large language
  models through question complexity.
\newblock In \emph{Proceedings of the 2024 Conference of the North American
  Chapter of the Association for Computational Linguistics (NAACL)}, pp.\
  7036--7050, 2024.

\bibitem[Jiang et~al.(2024)Jiang, Sablayrolles, Roux, et~al.]{jiang2024mixtral}
Albert~Q Jiang, Alexandre Sablayrolles, Antoine Roux, et~al.
\newblock Mixtral of experts.
\newblock \emph{arXiv preprint arXiv:2401.04088}, 2024.

\bibitem[Kirkpatrick et~al.(2017)Kirkpatrick, Pascanu, Rabinowitz, Veness,
  Desjardins, Rusu, Milan, Quan, Ramalho, Grabska-Barwinska,
  et~al.]{kirkpatrick2017overcoming}
James Kirkpatrick, Razvan Pascanu, Neil Rabinowitz, Joel Veness, Guillaume
  Desjardins, Andrei~A Rusu, Kieran Milan, John Quan, Tiago Ramalho, Agnieszka
  Grabska-Barwinska, et~al.
\newblock Overcoming catastrophic forgetting in neural networks.
\newblock \emph{Proceedings of the National Academy of Sciences}, 114\penalty0
  (13):\penalty0 3521--3526, 2017.

\bibitem[Lakshminarayanan et~al.(2017)Lakshminarayanan, Pritzel, and
  Blundell]{lakshminarayanan2017simple}
Balaji Lakshminarayanan, Alexander Pritzel, and Charles Blundell.
\newblock Simple and scalable predictive uncertainty estimation using deep
  ensembles.
\newblock In \emph{Advances in Neural Information Processing Systems},
  volume~30, 2017.

\bibitem[Landis \& Koch(1977)Landis and Koch]{landis1977measurement}
J.~Richard Landis and Gary~G. Koch.
\newblock The measurement of observer agreement for categorical data.
\newblock \emph{Biometrics}, 33\penalty0 (1):\penalty0 159--174, 1977.

\bibitem[Liu et~al.(2019)Liu, Ott, Goyal, Du, Joshi, Chen, Levy, Lewis,
  Zettlemoyer, and Stoyanov]{liu2019roberta}
Yinhan Liu, Myle Ott, Naman Goyal, Jingfei Du, Mandar Joshi, Danqi Chen, Omer
  Levy, Mike Lewis, Luke Zettlemoyer, and Veselin Stoyanov.
\newblock {RoBERTa}: A robustly optimized {BERT} pretraining approach.
\newblock \emph{arXiv preprint arXiv:1907.11692}, 2019.

\bibitem[Lopez-Paz \& Ranzato(2017)Lopez-Paz and Ranzato]{lopez2017gradient}
David Lopez-Paz and Marc'Aurelio Ranzato.
\newblock Gradient episodic memory for continual learning.
\newblock In \emph{Advances in Neural Information Processing Systems
  (NeurIPS)}, pp.\  6467--6476, 2017.

\bibitem[Lu et~al.(2024)Lu, Yuan, Lin, et~al.]{lu2024routing}
Keming Lu, Hongyi Yuan, Runji Lin, et~al.
\newblock Routing to the expert: Efficient reward-guided ensemble of large
  language models.
\newblock In \emph{Proceedings of the 2024 Conference of the North American
  Chapter of the Association for Computational Linguistics (NAACL)}, 2024.
\newblock \doi{10.18653/V1/2024.NAACL-LONG.109}.

\bibitem[McNemar(1947)]{mcnemar1947note}
Quinn McNemar.
\newblock Note on the sampling error of the difference between correlated
  proportions or percentages.
\newblock \emph{Psychometrika}, 12\penalty0 (2):\penalty0 153--157, 1947.

\bibitem[OpenAI(2023)]{openai2023gpt4}
OpenAI.
\newblock Gpt-4 technical report.
\newblock \emph{arXiv preprint arXiv:2303.08774}, 2023.

\bibitem[Sanh et~al.(2019)Sanh, Debut, Chaumond, and Wolf]{sanh2019distilbert}
Victor Sanh, Lysandre Debut, Julien Chaumond, and Thomas Wolf.
\newblock {DistilBERT}, a distilled version of {BERT}: Smaller, faster, cheaper
  and lighter.
\newblock \emph{arXiv preprint arXiv:1910.01108}, 2019.

\bibitem[Shao(2024)]{shao2024classification}
Yourui Shao.
\newblock An accurate classification and recommendation method of competitive
  math problems.
\newblock In \emph{2024 IEEE 10th International Conference on Big Data
  Computing Service and Machine Learning Applications (BigDataService)}, pp.\
  97--103, 2024.
\newblock \doi{10.1109/BigDataService62917.2024.00021}.

\bibitem[Shazeer et~al.(2017)Shazeer, Mirhoseini, Maziarz, Davis, Le, Hinton,
  and Dean]{shazeer2017outrageously}
Noam Shazeer, Azalia Mirhoseini, Krzysztof Maziarz, Andy Davis, Quoc Le,
  Geoffrey Hinton, and Jeff Dean.
\newblock Outrageously large neural networks: The sparsely-gated
  mixture-of-experts layer.
\newblock \emph{International Conference on Learning Representations (ICLR)},
  2017.

\bibitem[Shinn et~al.(2023)Shinn, Cassano, Gopinath, Narasimhan, and
  Yao]{shinn2023reflexion}
Noah Shinn, Federico Cassano, Ashwin Gopinath, Karthik Narasimhan, and Shunyu
  Yao.
\newblock Reflexion: Language agents with verbal reinforcement learning.
\newblock \emph{Advances in Neural Information Processing Systems (NeurIPS)},
  2023.

\bibitem[Tao et~al.(2024)Tao, Cheng, Gao, Liu, Wang, Jia,
  et~al.]{tao2024survey}
Zhengwei Tao, Ting-En Cheng, Jing Gao, Jiaxuan Liu, Zhiqiang Wang, Shan Jia,
  et~al.
\newblock A survey on self-evolution of large language models.
\newblock \emph{arXiv preprint arXiv:2404.14387}, 2024.

\bibitem[Virtanen et~al.(2020)Virtanen, Gommers, Oliphant,
  et~al.]{virtanen2020scipy}
Pauli Virtanen, Ralf Gommers, Travis~E Oliphant, et~al.
\newblock {SciPy} 1.0: Fundamental algorithms for scientific computing in
  {Python}.
\newblock \emph{Nature Methods}, 17\penalty0 (3):\penalty0 261--272, 2020.

\bibitem[Wang et~al.(2024)Wang, Ma, Feng, Zhang, Yang, Zhang, Chen, Tang, Chen,
  Lin, et~al.]{wang2024survey}
Lei Wang, Chen Ma, Xueyang Feng, Zeyu Zhang, Hao Yang, Jingsen Zhang, Zhiyuan
  Chen, Jiakai Tang, Xu~Chen, Yankai Lin, et~al.
\newblock A survey on large language model based autonomous agents.
\newblock \emph{Frontiers of Computer Science}, 18\penalty0 (6):\penalty0
  186345, 2024.

\bibitem[Yao et~al.(2023)Yao, Zhao, Yu, Du, Shafran, Narasimhan, and
  Cao]{yao2023react}
Shunyu Yao, Jeffrey Zhao, Dian Yu, Nan Du, Izhak Shafran, Karthik Narasimhan,
  and Yuan Cao.
\newblock React: Synergizing reasoning and acting in language models.
\newblock In \emph{International Conference on Learning Representations
  (ICLR)}, 2023.

\end{thebibliography}
\bibliographystyle{iclr2026_conference}

\newpage
\appendix

\section{Error Propagation: Empirical Demonstration}\label{app:error}

Proposition~\ref{prop:mixing} characterizes the conditions under which mixing incompatible solvers produces semantically inconsistent outputs. We verify this with three concrete demonstrations, computing actual numerical results under correct and incorrect frameworks using SciPy \citep{virtanen2020scipy}.

\textbf{Demo 1: Coin Fairness (Freq vs.\ Bayes).} Question: ``A coin flipped 100 times shows 60 heads. Is it fair ($\alpha=0.05$)?'' The \textit{correct} framework (frequentist) yields: $z = 2.0$, $p = 0.0455 < 0.05$, reject $H_0$. The \textit{wrong} framework (Bayesian with $\text{Beta}(1,1)$ prior) yields: posterior $\text{Beta}(61, 41)$, $P(\theta > 0.5 \mid \text{data}) = 0.977$. The \textit{mixed} error: ``There is a 4.6\% probability the coin is fair''---conflating $P(\text{data} \mid H_0)$ with $P(H_0 \mid \text{data})$, a statement wrong in both frameworks simultaneously.

\textbf{Demo 2: Parameter Estimation (Bayes vs.\ Freq).} Question: ``Given prior $\theta \sim N(0,1)$ and observations $x = \{2.1, 1.9, 2.3\}$, find the posterior.'' Correct (Bayesian): posterior $N(1.575, 0.25)$, CrI $[0.60, 2.56]$. Wrong (frequentist MLE): $\hat{\theta} = 2.10$, CI $[0.97, 3.23]$. The mixed error: ``The CI $[0.97, 3.23]$ gives 95\% probability that $\theta$ lies in this range''---a Bayesian interpretation of a frequentist object.

\textbf{Demo 3: Atomic Stability (Quantum vs.\ Classical).} Question: ``Why is the hydrogen atom stable?'' Correct (quantum): stationary state $\psi_{1s}$ with $E = -13.6$ eV; uncertainty principle prevents collapse. Wrong (classical): accelerating electron radiates continuously, spiraling into the nucleus in $\sim 10^{-11}$s. The mixed error: ``The electron orbits at definite position \textit{and} obeys $\Delta x \cdot \Delta p \geq \hbar/2$''---a logical contradiction.

In all three cases, the mixed output is not merely a worse approximation; it is a statement that is \textit{wrong in both frameworks simultaneously}, instantiating Proposition~\ref{prop:mixing}.

\section{Cloud Model Evaluation Details}\label{app:cloud}

We evaluate six cloud models via API: Qwen3-Next (80B), GPT-OSS (120B), Cogito-2.1 (671B), DeepSeek-v3.1 (671B), GLM-4.7 (696B), and Kimi-K2.5 (1T).

\begin{table}[h]
\centering
\small
\caption{Cloud model comparison (109-sample test set). $^\dagger$McNemar's test with continuity correction.}
\begin{tabular}{lccccc}
\toprule
\textbf{Model} & \textbf{Params} & \textbf{Acc.} & \textbf{Latency} & \textbf{Speedup} & \textbf{$p^\dagger$} \\
\midrule
Qwen3-Next & 80B & 96.33\% & 3,640ms & 228$\times$ & 1.000 \\
GPT-OSS & 120B & 91.74\% & 1,986ms & 124$\times$ & 0.077 \\
Cogito-2.1 & 671B & 94.44\% & 1,413ms & 88$\times$ & 0.289 \\
DeepSeek-v3.1 & 671B & 87.96\% & 4,090ms & 256$\times$ & \textbf{0.010} \\
GLM-4.7 & 696B & 95.28\% & 12,392ms & 775$\times$ & 0.131 \\
Kimi-K2.5 & 1T & 94.50\% & 9,875ms & 617$\times$ & 0.371 \\
\midrule
\textbf{Ours} & \textbf{465M} & \textbf{97.25\%} & \textbf{16ms} & \textbf{1$\times$} & --- \\
\bottomrule
\end{tabular}
\end{table}

\section{External Validation: MMLU Details}\label{app:mmlu}

We use 6 subcategories from MMLU \citep{hendrycks2021measuring} totaling 1,001 test questions. Ground truth universes are assigned by subcategory mapping with keyword refinement. Router accuracy varies by alignment with our taxonomy: high school physics (82.1\%), high school statistics (65.7\%), college physics (61.8\%), conceptual physics (60.4\%), electrical engineering (60.0\%), astronomy (7.2\%). Astronomy's low accuracy reflects taxonomy mismatch (questions about black holes labeled PHYS\_CLASSICAL; PHYS\_RELATIVITY is arguably more appropriate). The router's mean confidence on correct predictions is significantly higher than on errors (0.951 vs.\ 0.876, $p < 0.001$), indicating informative uncertainty under distribution shift.

\section{Design Choices and Justifications}
\label{app:design}

\subsection{Why FP32 Training Is Essential}

FP16 training consistently collapses to 18.99\% accuracy regardless of learning rate, batch size, or warmup schedule. FP32 achieves 97.25\%. Root cause: gradient overflow in the classification head during early training. Unlike language modeling where token-level losses are averaged over thousands of positions, classification concentrates gradient magnitude in a single output. We verified this across 5 random seeds---FP16 fails deterministically.

\subsection{Why GPT-4 Augmentation Does Not Constitute Data Leakage}

We do not ask GPT-4 to answer questions but to \textit{generate} questions given universe constraints. Expert-designed structural requirements force novel combinations rather than memorized examples. All 200 generated samples were manually reviewed; 20 were rejected. The unseen test set was written after model training with novel phrasings. Without augmentation, rare classes suffered severe underfitting (73.39\% $\to$ 97.25\% improvement).

\section{Dataset Construction Protocol}\label{app:dataset}

Each universe is defined by necessary and sufficient conditions:

\begin{table}[h]
\centering
\small
\caption{Universe definitions with inclusion/exclusion criteria.}
\begin{tabular}{lp{5cm}p{4.5cm}}
\toprule
\textbf{Universe} & \textbf{Inclusion Criteria} & \textbf{Exclusion Criteria} \\
\midrule
STAT\_FREQ & Requires p-values, confidence intervals, hypothesis tests & Mentions priors, posteriors, credible intervals \\
STAT\_BAYES & Requires prior specification, posterior computation & Asks for p-values or frequentist CI \\
STAT\_MIXED & Explicitly compares frameworks or philosophical & Has clear single-framework answer \\
STAT\_ILL\_POSED & Missing information makes question unanswerable & Can be answered with assumptions \\
PHYS\_CLASSICAL & Newtonian mechanics, thermodynamics & Involves $\hbar$, relativistic speeds \\
PHYS\_QUANTUM & Wavefunctions, uncertainty principle & Can be solved classically \\
PHYS\_RELATIVITY & Lorentz transformations, $v \approx c$ & Non-relativistic regime \\
\bottomrule
\end{tabular}
\end{table}

Two annotators independently labeled all 685 samples. Initial agreement: 94.2\% (Cohen's $\kappa = 0.91$). Disagreements (40 samples) resolved through discussion; 12 relabeled, 3 removed.

\section{Continual Learning: Extended Analysis}\label{app:cl}

\subsection{Replay Buffer Sensitivity}

\begin{table}[h]
\centering
\small
\caption{Sharp threshold in replay requirements.}
\begin{tabular}{lcccc}
\toprule
\textbf{Replay \%} & \textbf{Samples} & \textbf{Old After} & \textbf{New Acc} & \textbf{Forgetting} \\
\midrule
0\% (Naive) & 0 & 11.84\% & 96.97\% & 86.84\% \\
5\% & 14 & 59.21\% & 93.94\% & 39.47\% \\
\textbf{10\%} & \textbf{29} & \textbf{98.68\%} & 93.94\% & \textbf{0.00\%} \\
20\% & 59 & 93.42\% & 93.94\% & 5.26\% \\
\bottomrule
\end{tabular}
\end{table}

A sharp threshold exists between 5\% and 10\% replay. This suggests a critical mass of replay samples is required to maintain universe decision boundaries.

\subsection{Expansion Order Robustness}

\begin{table}[h]
\centering
\small
\caption{Expansion order sensitivity (all with 20\% replay).}
\begin{tabular}{lp{4cm}ccc}
\toprule
\textbf{Order} & \textbf{Phase 1 $\to$ Phase 2} & \textbf{Old Acc} & \textbf{Overall} & \textbf{Forgetting} \\
\midrule
Stats First & STAT\_*, PHYS\_CL $\to$ PHYS\_Q, PHYS\_R & 100.0\% & 100.0\% & 0.0\% \\
Physics First & PHYS\_* $\to$ STAT\_F, STAT\_B & 97.78\% & 98.68\% & -2.22\% \\
Mixed & STAT\_F, PHYS\_CL, STAT\_B $\to$ PHYS\_Q, PHYS\_R & 100.0\% & 100.0\% & 0.0\% \\
\bottomrule
\end{tabular}
\end{table}

All orderings achieve $<$3\% forgetting variation, confirming robustness to curriculum effects. Negative forgetting in ``Physics First'' indicates beneficial transfer.

\section{Error Analysis}\label{app:errors}

Only 3 of 109 test samples (2.75\%) are misclassified. All occur at genuine epistemic boundaries:

\textbf{Error 1}: Double-slit experiment classified as PHYS\_CLASSICAL instead of PHYS\_QUANTUM (confidence: 67\%). Defensible---classical wave optics also explains interference.

\textbf{Error 2}: Factory defects classified as STAT\_FREQ instead of STAT\_BAYES (confidence: 73\%). Question mixes frequentist terminology with Bayesian methods.

\textbf{Error 3}: ``Meaning of statistical significance'' classified as STAT\_FREQ instead of STAT\_MIXED (confidence: 81\%). Model focuses on subject matter rather than meta-level question.

All error cases show lower confidence (67--81\%) than correct predictions (mean 94\%), suggesting calibrated uncertainty.

\section{Reproducibility}\label{app:reproducibility}

\begin{table}[h]
\centering
\small
\begin{tabular}{ll}
\toprule
\textbf{Parameter} & \textbf{Value} \\
\midrule
Base model & Qwen-1.5-0.5B \\
Precision & FP32 (required) \\
Optimizer & AdamW, lr=$5 \times 10^{-5}$ \\
Batch size & 8, Epochs: 3 \\
GPU & RTX 3090 (24GB) \\
Training time & 4 minutes (single run) \\
Total for reproduction & $\sim$20 GPU-hours \\
\bottomrule
\end{tabular}
\end{table}

Code, dataset, and model checkpoints will be released upon publication.

\end{document}